\documentclass{llncs}
\usepackage{amsmath}
\usepackage{amssymb}
\usepackage{times}
\usepackage{indentfirst}
\usepackage{graphicx}

\begin{document}

\title{Characteristic of partition-circuit matroid through approximation number}

\author{Yanfang Liu, William Zhu~\thanks{Corresponding author.
E-mail: williamfengzhu@gmail.com (William Zhu)}, Yuanping Zhang}
\institute{
Lab of Granular Computing,\\
Zhangzhou Normal University, Zhangzhou 363000, China}



\date{\today}          
\maketitle

\begin{abstract}
Rough set theory is a useful tool to deal with uncertain, granular and incomplete knowledge in information systems.
And it is based on equivalence relations or partitions.
Matroid theory is a structure that generalizes linear independence in vector spaces, and has a variety of applications in many fields.
In this paper, we propose a new type of matroids, namely, partition-circuit matroids, which are induced by partitions.
Firstly, a partition satisfies circuit axioms in matroid theory, then it can induce a matroid which is called a partition-circuit matroid.
A partition and an equivalence relation on the same universe are one-to-one corresponding, then some characteristics of partition-circuit matroids are studied through rough sets.
Secondly, similar to the upper approximation number which is proposed by Wang and Zhu, we define the lower approximation number.
Some characteristics of partition-circuit matroids and the dual matroids of them are investigated through the lower approximation number and the upper approximation number.

\end{abstract}


\section{Introduction}
Rough set theory~\cite{Pawlak82Rough} was first proposed by Pawlak as a mathematical tool for imperfect data analysis.
Through two definable subsets which are the lower and upper approximations, rough set theory deals with the approximation of an arbitrary subset of a universe, and discovers rules from data without priori knowledge.
It has been successfully applied into feature selection~\cite{ChenMiaoWangWu11ARough,HuLiuYu07Mixed,HuYuLiuWu08Neighborhood}, rule extraction~\cite{InuiguchiTanino02Generalized,DuHuZhuMa11Rule}, uncertainty reasoning~\cite{Pawlak91Rough,PawlakSkowron07RoughSetAndBooleanReasoning}, decision evaluation~\cite{WuLeung11theoryandapplications,QianLiangLiZhangDang08Measures}, granular computing~\cite{SkowronStepaniukSwiniarski12Modeling,YaoZhong99Potential,Yao08Granular,Yao01Modeling} and so on.

Rough set theory is based on equivalence relations or partitions.
In recent years, many extensions of classical rough sets have been proposed.
For example, through relaxing the partitions to coverings, several covering-based rough sets have been proposed~\cite{BonikowskiBryniarskiWybraniecSkardowska98Extensions,ZhuWang03Reduction,Pomykala87Approximation,Zakowski83Approximations,Zhu09RelationshipAmong}.
In these covering-based rough sets, the concept of covering of a universe was presented to construct the lower and upper approximations of an arbitrary set.

Matroid theory~\cite{Lai01Matroid} was proposed by Whitney as a generalization of linear independence in vector spaces.
It borrowed extensively from the terminology of linear algebra and graph theory, and made great progress in recent decades.
In theory, a matroid can be defined by dozens of ways, which provides a well platform to connect it with other theories.
Some authors have connected matroid theory with classical rough sets~\cite{LiuZhuZhang12Relationshipbetween}, covering-based rough sets~\cite{ZhangWangFengFeng11reductionofrough,WangZhu11Matroidal,WangZhuMin11Transversal}, generalized rough sets based on relations~\cite{WangZhuMin11TheVectorially,ZhuWang11Matroidal}, fuzzy theory~\cite{RoyVoxman88Fuzzy} and so on.
In application, it has widely been used in many fields, such as combinatorial optimization~\cite{Lawler01Combinatorialoptimization} and algorithm design~\cite{Edmonds71Matroids}.

In this paper, a new type of matroids is proposed, which is called partition-circuit matroid.
Through partitions, we build the connection between rough sets with matroids.
And some characteristics of partition-circuit matroids and the dual matroids of them are studied through rough sets and the lower and upper approximation numbers.
The main contributions of this paper are two folds.
On the one hand, a partition satisfies the circuit axioms of matroid theory~\cite{Lai01Matroid}, then a type of matroids, called partition-circuit matroids, are induced by partitions.
Any partition has a one-to-one correspondence with an equivalence relation on the same universe.
Therefore, some characteristics of partition-circuit matroids are studied through the lower and upper approximation operators in rough sets.
On the other hand, we investigate some characteristics of partition-circuit matroids and the dual matroids of them through lower and upper approximation numbers.
Wang and Zhu proposed the upper approximation number based on coverings in~\cite{WangZhu11Matroidal,WangZhuMin11Transversal,ZhuWang11Matroidal}.
Similar to the upper approximation number, we propose the lower approximation number based on coverings.
First, some properties of the lower and upper approximation numbers are investigated.
Second, some characteristics of partition-circuit matroids are studied through the lower approximation number based on partitions.
Third, some characteristics of the dual matroid of a partition-circuit one are investigated through the upper approximation number.

The rest of this paper is arranged as follows. Section~\ref{S:basicdefinitions} introduces some basic definitions of rough sets and matroids.
In Section~\ref{S:matroidalstructure}, through partitions, we get a type of matroids, namely, partition-circuit matroids, and study them through rough sets.
Section~\ref{S:partition-circuitanddualmatroid} introduces the upper approximation number and proposes the lower upper approximation number to investigate some characteristics of partition-circuit matroids and the dual matroids of them.
Finally, we conclude this paper in Section~\ref{S:conclusions}.

\section{Basic definitions}
\label{S:basicdefinitions}
In this section, we first review some fundamental ideas and results related to rough sets, and then introduce some notions of matroids.

\subsection{The rough set model}
Let $U$ be a finite and nonempty set and $R$ an equivalence relation on $U$.
The equivalence relation $R$ induces a partition $U/R=\{P_{1}, \cdots, P_{m}\}$ on $U$, where $P_{1}, \cdots, P_{m}$ are the equivalence classes.

The equivalence classes of $R$ are elementary sets to construct rough set approximations.
For any $X\subseteq U$, its lower and upper approximations are defined as follows:
\begin{center}
~$\underline{R}(X)=\bigcup\{P\in U/R:P\subseteq X\}$;\\
\quad\quad $\overline{R}(X)=\bigcup\{P\in U/R:P\bigcap X\neq\emptyset\}$.
\end{center}

$X^{c}$ is denoted by the complement of $X$ in $U$ and $Y\subseteq U$. We have the following properties of rough sets:

(1L) $\underline{R}(X)\subseteq X$;

(1H) $X\subseteq \overline{R}(X)$;

(2L) $\underline{R}(X\bigcap Y)=\underline{R}(X)\bigcap \underline{R}(Y)$;

(2H) $\overline{R}(X\bigcup Y)=\overline{R}(X)\bigcup \overline{R}(Y)$;

(3L) $\underline{R}(X^{c})=(\overline{R}(X))^{c}$;

(3H) $\overline{R}(X^{c})=(\underline{R}(X))^{c}$;

(4L) $\underline{R}((\underline{R}(X))^{c})=(\underline{R}(X))^{c}$;

(4H) $\overline{R}((\overline{R}(X))^{c})=(\overline{R}(X))^{c}$.

The (1L), (1H), (2L), (2H), (4L) and (4H) are characteristic properties of the lower and upper approximation operators~\cite{LinLiu94Rough}, in other words, all other properties can be deduced from these properties.

\subsection{The matroid model}
A matroid is a structure that generalizes the notion of linear independence in matrices.
There are several ways to define a matroid, such as independent sets, circuits, bases, rank function and closure operator.
We will first define a matroid that focuses on its independent sets.

\begin{definition}(Matroid~\cite{Lai01Matroid})
A matroid $M$ is a pair $(U, \mathbf{I})$, where $U$ is a finite set (called the ground set) and $\mathbf{I}$ is a collection of subsets of $U$ (called the independent sets) with the following properties:\\
(I1) $\emptyset\in\mathbf{I}$;\\
(I2) If $I\in\mathbf{I}$ and $I'\subseteq I$, then $I'\in\mathbf{I}$;\\
(I3) If $I_{1}, I_{2}\in\mathbf{I}$ and $|I_{1}|<|I_{2}|$, then there exists $u\in I_{2}-I_{1}$ such that $I_{1}\bigcup\{u\}\in\mathbf{I}$, where $|I|$ denotes the cardinality of $I$.
\end{definition}

In order to make some expressions simple, we introduce some symbols as follows.

\begin{definition}(\cite{Lai01Matroid})
Let $U$ be a finite set and $\mathbf{A}$ a family of subsets of $U$.
Then\\
$Upp(\mathbf{A})=\{X\subseteq U:\exists A\in\mathbf{A}, s.t. A\subseteq X\}$;\\
$Low(\mathbf{A})=\{X\subseteq U:\exists A\in\mathbf{A}, s.t. X\subseteq A\}$;\\
$Max(\mathbf{A})=\{X\in\mathbf{A}:\forall Y\in\mathbf{A}, X\subseteq Y\Rightarrow X=Y\}$;\\
$Min(\mathbf{A})=\{X\in\mathbf{A}:\forall Y\in\mathbf{A}, Y\subseteq X\Rightarrow X=Y\}$;\\
$Opp(\mathbf{A})=\{X\subseteq U:X\notin\mathbf{A}\}$.
\end{definition}

A maximal independent set of a matroid is a base.
The base of a matroid generalizes the maximal linear independence in vector spaces.

\begin{definition}(Base~\cite{Lai01Matroid})
Let $M=(U, \mathbf{I})$ be a matroid.
Any maximal independent set in $M$ is called a base of $M$, and the family of bases of $M$ is denoted by $\mathbf{B}(M)$, i.e., $\mathbf{B}(M)=Max(\mathbf{I})$.
\end{definition}

A matroid and its family of bases are uniquely determined by each other.
Then, one of equivalent definitions of a matroid is represented in terms of bases.

\begin{proposition}(Base axioms~\cite{Lai01Matroid})
Let $\mathbf{B}$ be a family of subsets of $U$.
Then there exists a matroid $M=(U, \mathbf{I})$ such that $\mathbf{B}=\mathbf{B}(M)$ if and only if $\mathbf{B}$ satisfies the following two conditions:\\
(B1) $\mathbf{B}\neq\emptyset$;\\
(B2) If $B_{1}, B_{2}\in\mathbf{B}$ and $x\in B_{1}-B_{2}$, then there exists $y\in B_{2}-B_{1}$ such that $(B_{1}-\{x\})\bigcup\{y\}\in\mathbf{B}$.
\end{proposition}

The complement of the independent sets in power sets are dependent ones.
And a minimal set of the dependent sets is called a circuit of the matroid.

\begin{definition}(Circuit~\cite{Lai01Matroid})
Let $M=(U, \mathbf{I})$ be a matroid.
A minimal dependent set in $M$ is called a circuit of $M$, and we denote the family of all circuits of $M$ by $\mathbf{C}(M)$, i.e., $\mathbf{C}(M)=Min(\mathbf{I}^{c})$, where $\mathbf{I}^{c}$ denotes the complement of $\mathbf{I}$ in $2^{U}$.
\end{definition}

A matroid can be defined from the viewpoint of its circuits in the following proposition.
A matroid uniquely determines its circuits, and vice versa.

\begin{proposition}(Circuit axioms~\cite{Lai01Matroid})
\label{P:circuitaxioms}
Let $\mathbf{C}$ be a family of subsets of $U$.
Then there exists a matroid $M=(U, \mathbf{I})$ such that $\mathbf{C}=\mathbf{C}(M)$ if and only if $\mathbf{C}$ satisfies the following properties:\\
(C1) $\emptyset\notin\mathbf{C}$;\\
(C2) If $C, C'\in\mathbf{C}$ and $C'\subseteq C$, then $C'=C$;\\
(C3) If $C_{1}, C_{2}\in\mathbf{C}, C_{1}\neq C_{2}$ and $u\in C_{1}\bigcap C_{2}$, then there exists $C_{3}\in\mathbf{C}$ such that $C_{3}\subseteq C_{1}\bigcup C_{2}-\{u\}$.
\end{proposition}

The cardinality of any maximal independent set in vector spaces can be expressed by the rank function of a matroid.

\begin{definition}(Rank function~\cite{Lai01Matroid})
Let $M=(U, \mathbf{I})$ be a matroid.
Then $r_{M}$ is called the rank function of $M$, where $r_{M}(X)=max\{|I|:I\subseteq X, I\in\mathbf{I}\}$ for all $X\subseteq U$.
\end{definition}

One of equivalent definitions of a matroid is represented in terms of its rank function.
And a matroid and its rank function are uniquely determined by each other.

\begin{proposition}
\label{P:rankdeterminematroid}
Let $M=(U, \mathbf{I})$ be a matroid and $r_{M}$ its rank function.
Then for all $X\subseteq U$, $X\in\mathbf{I}$ iff $r_{M}(X)=|X|$.
\end{proposition}

In order to represent the relationship between an element and a set of a universe, we introduce the closure operator through the rank function in matroids.

\begin{definition}(Closure~\cite{Lai01Matroid})
\label{D:closure}
Let $M=(U, \mathbf{I})$ be a matroid and $X\subseteq U$.
For any $u\in U$, if $r_{M}(X)=r_{M}(X\bigcup\{u\})$, then $u$ depends on $X$.
The subset including all elements depending on $X$ of $U$ is called the closure with respect to $X$ and denoted by $cl_{M}(X)$:\\
\centerline{$cl_{M}(X)=\{u\in U:r_{M}(X)=r_{M}(X\bigcup\{u\})\}$,}
where $cl_{M}$ is called the closure operator of $M$.
\end{definition}

In a matroid, if the closure of a set is equal to itself, then the set is a closed set.

\begin{definition}(Closed set~\cite{Lai01Matroid})
\label{D:closedset}
Let $M=(U, \mathbf{I})$ be a matroid and $X\subseteq U$.
Therefore $X$ is a closed set of $M$ if $cl_{M}(X)=X$.
\end{definition}

Duality is one of important characteristics of matroids, which can generate a new matroid through a given matroid.
And the new matroid is called dual matroid which generalizes the orthogonal complement of a vector space.
The dual matroid of a matroid is introduced as follows.

\begin{definition}(Dual matroid~\cite{Lai01Matroid})
Let $M=(U, \mathbf{I})$ be a matroid.
Then the matroid whose base family is $\{B:B^{c}\in\mathbf{B}(M)\}$ is called the dual matroid of $M$ and denoted by $M^{*}$.
\end{definition}

The relationship between a matroid and its dual matroid is represented as follows.

\begin{proposition}
\label{P:relationshipbetweenmatroidanditsdual}
Let $M=(U, \mathbf{I})$ be a matroid and $M^{*}$ the dual matroid.
$r_{M}$ is the rank function of $M$ and $r^{*}_{M}$ is the rank function of $M^{*}$.
Then, for all $X\subseteq U$,\\
\centerline{$r^{*}_{M}(X)=|X|-r_{M}(U)+r_{M}(X^{c})$.}
\end{proposition}

Through the independent set axioms of a matroid, Liu and Chen~\cite{LiuChen94Matroid} proposed partition matroids induced by partitions with respect to a group of nonnegative integers.

\begin{definition}(Partition matroid~\cite{LiuChen94Matroid})
\label{D:partitionmatroid}
Let $\mathbf{P}=\{P_{1}, \cdots,$ $ P_{m}\}$ be a partition on $U$ and $k_{1}, \cdots, k_{m}$ be a group of nonnegative integers.
Then $M(\mathbf{P}; k_{1}, \cdots, k_{m})=(U, \mathbf{I}(\mathbf{P}; k_{1},$ $ \cdots, k_{m}))$ is a matroid where $\mathbf{I}(\mathbf{P}; k_{1}, \cdots$, $ k_{m})=\{X\subseteq U: |X\bigcap P_{i}|\leq k_{i}, 1\leq i\leq m\}$, and it is called the partition matroid induced by $\mathbf{P}$ with respect to $k_{1}, \cdots, k_{m}$.
\end{definition}

\section{Matroidal structure of partition}
\label{S:matroidalstructure}
In this section, we propose a matroid which is induced by a partition through the circuit axioms.

\begin{definition}
\label{D:matroidinducedbypartition}
Let $\mathbf{P}$ be a partition on $U$. We define a family $\mathbf{C}_{\mathbf{P}}$ of subsets of $U$, where $\mathbf{C}_{\mathbf{P}}=\mathbf{P}$.
\end{definition}

According to the above definition, a family of subsets of a universe is equal to a partition on the universe.
In fact, the family of subsets of the universe satisfies the circuit axioms of matroids.

\begin{proposition}
\label{P:Psatisfiescircuitaxoms}
Let $\mathbf{P}$ be a partition on $U$. Then $\mathbf{C}_{\mathbf{P}}$ satisfies (C1), (C2) and (C3).
\end{proposition}

\begin{proof}
According to Definition~\ref{D:matroidinducedbypartition} and Proposition~\ref{P:circuitaxioms}, it is straightforward.
\end{proof}

According to Proposition~\ref{P:circuitaxioms}, a matroid and its circuits determine each other.
Therefore, the family of subsets of the universe induced by a partition can generate a matroid.

\begin{definition}(Partition-circuit matroid)
Let $\mathbf{P}$ be a partition on $U$.
The matroid, whose family of all circuits is $\mathbf{C}_{\mathbf{P}}$, is denoted by $M_{\mathbf{P}}=(U, \mathbf{I}_{\mathbf{P}})$ and called partition-circuit matroid.
We say $M_{\mathbf{P}}=(U, \mathbf{I}_{\mathbf{P}})$ is the matroid induced by $\mathbf{P}$, where $\mathbf{I}_{\mathbf{P}}=Opp(Upp(\mathbf{C}_{\mathbf{P}}))$.
\end{definition}

According to Definition~\ref{D:matroidinducedbypartition}, the family of all circuits of a matroid is equal to a partition on the same universe.
Then we will represent the independent sets of the matroid from the viewpoint of the partition.

\begin{proposition}
\label{P:representedbypartitionmatroid}
Let $\mathbf{P}$ be a partition on $U$ and $M_{\mathbf{P}}=(U, \mathbf{I}_{\mathbf{P}})$ the partition-circuit matroid induced by $\mathbf{P}$.
Then $\mathbf{I}_{\mathbf{P}}=\{X\subseteq U:\forall P\in\mathbf{P}, |X\bigcap P|\leq |P|-1\}$.
\end{proposition}

\begin{proof}
We only need to prove $Opp(Upp(\mathbf{C}_{\mathbf{P}}))=\{X\subseteq U:\forall P\in\mathbf{P}, |X\bigcap P|\leq |P|-1\}$.
For all $X\notin\{X\subseteq U:\forall P\in\mathbf{P}, |X\bigcap P|\leq |P|-1\}$, then there exists $P\in\mathbf{P}$ such that $|X\bigcap P|\geq |P|$.
And $X\bigcap P\subseteq P$, then $X\bigcap P=P$, i.e., $P\subseteq X$.
According to Definition~\ref{D:matroidinducedbypartition}, $\mathbf{C}_{\mathbf{P}}=\mathbf{P}$.
Therefore $X\in Upp(\mathbf{C}_{\mathbf{P}})$, i.e., $X\notin Opp(Upp(\mathbf{C}_{\mathbf{P}}))$.
This proves $Opp(Upp(\mathbf{C}_{\mathbf{P}}))\subseteq\{X\subseteq U:\forall P\in\mathbf{P}, |X\bigcap P|\leq |P|-1\}$.
Conversely, for all $X\notin Opp(Upp(\mathbf{C}_{\mathbf{P}}))$, i.e., $X\in Upp(\mathbf{C}_{\mathbf{P}})$, according to Definition~\ref{D:matroidinducedbypartition}, there exists $P\in\mathbf{P}$ such that $P\subseteq X$.
Thus $|X\bigcap P|=|P|$.
So $X\notin\{X\subseteq U:\forall P\in\mathbf{P}, |X\bigcap P|\leq |P|-1\}$.
This proves that $Opp(Upp(\mathbf{C}_{\mathbf{P}}))\supseteq\{X\subseteq U:\forall P\in\mathbf{P}, |X\bigcap P|\leq |P|-1\}$.
To sum up, this completes the proof.
\end{proof}

According to Proposition~\ref{P:representedbypartitionmatroid}, the independent sets of a partition-circuit matroid can be expressed by a group of nonnegative integers.
Then we obtain a proposition in the the following.

\begin{proposition}
\label{P:partitoin-circuitmatroidispartitionone}
Let $\mathbf{P}$ be a partition on $U$ and $M_{\mathbf{P}}=(U, \mathbf{I}_{\mathbf{P}})$ the partition-circuit matroid induced by $\mathbf{P}$.
Then, $M_{\mathbf{P}}$ is a partition matroid.
\end{proposition}

\begin{proof}
According to Definition~\ref{D:partitionmatroid} and Proposition~\ref{P:representedbypartitionmatroid}, it is straigtforward.
\end{proof}

Proposition~\ref{P:Psatisfiescircuitaxoms} establishes a matroidal structure of a partition.
A partition coincides with an equivalence relation.
Then, some characteristics of the matroid induced by a partition are represented through rough sets.
Firstly, the independent sets of the matroid are expressed by approximation operators of rough sets.

\begin{proposition}
\label{P:independentsetrepresentedbylower}
Let $\mathbf{P}$ be a partition on $U$ and $M_{\mathbf{P}}=(U, \mathbf{I}_{\mathbf{P}})$ the partition-circuit matroid induced by $\mathbf{P}$.
Let $R$ be an equivalence relation on $U$ and $\mathbf{P}=U/R$.
Then $\mathbf{I}_{\mathbf{P}}=\{X\subseteq U:\underline{R}(X)=\emptyset\}$.
\end{proposition}

\begin{proof}
We only need to prove that $Opp(Upp(\mathbf{C}_{\mathbf{P}}))=\{X\subseteq U:\underline R(X)=\emptyset\}$.
For all $X\notin\{X\subseteq U:\underline{R}(X)=\emptyset\}$, there exists $P\in U/R=\mathbf{P}$ such that $P\subseteq X$.
According to Definition~\ref{D:matroidinducedbypartition}, $\mathbf{C}_{\mathbf{P}}=\mathbf{P}$.
Therefore, $X\in Upp(\mathbf{C}_{\mathbf{P}})$, i.e., $X\notin Opp(Upp(\mathbf{C}_{\mathbf{P}}))$.
This proves that $Opp(Upp(\mathbf{C}_{\mathbf{P}}))\subseteq\{X\subseteq U:\underline{R}(X)=\emptyset\}$.
Conversely, for all $X\notin Opp(Upp(\mathbf{C}_{\mathbf{P}}))$, i.e., $X\in Upp(\mathbf{C}_{\mathbf{P}})$, there exists $P\in\mathbf{C}_{\mathbf{P}}$ such that $P\subseteq X$.
Therefore $\underline{R}(X)\neq\emptyset$, i.e., $X\notin\{X\subseteq U:\underline{R}(X)=\emptyset\}$.
This proves $Opp(Upp(\mathbf{C}_{\mathbf{P}}))\supseteq\{X\subseteq U:\underline{R}(X)=\emptyset\}$.
To sum up, this completes the proof.
\end{proof}

\begin{corollary}
Let $\mathbf{P}$ be a partition on $U$ and $M_{\mathbf{P}}=(U, \mathbf{I}_{\mathbf{P}})$ the partition-circuit matroid induced by $\mathbf{P}$.
Let $R$ be an equivalence relation on $U$ and $\mathbf{P}=U/R$.
Then $\mathbf{I}_{\mathbf{P}}=\{X\subseteq U:\overline{R}(X^{c})=U\}$.
\end{corollary}

\section{Partition-circuit matroid and its dual matroid}
\label{S:partition-circuitanddualmatroid}
Section~\ref{S:matroidalstructure} establishes a matroidal structure of a partition.
The matroid induced by a partition is called a partition-circuit matroid.
In order to investigate the characteristics of partition-circuit matroids and the dual matroids of them, we propose the lower approximation number and introduce the upper approximation number~\cite{WangZhu11Matroidal,WangZhuMin11Transversal,ZhuWang11Matroidal}.

\subsection{Lower approximation number and upper approximation number}
The upper approximation number proposed in~\cite{WangZhu11Matroidal,WangZhuMin11Transversal,ZhuWang11Matroidal} is based on coverings.
Similarly, we propose the lower approximation number.

\begin{definition}(Upper approximation number~\cite{WangZhu11Matroidal,WangZhuMin11Transversal,ZhuWang11Matroidal}, lower approximation number)
\label{D:lowerandupperapproximationnumbers}
Let $\mathbf{C}$ be a covering on $U$.
For all $X\subseteq U$,
\begin{center}
$\underline{f}_{\mathbf{C}}(X)=|\{K\in\mathbf{C}:K\subseteq X\}|$;\\
~~~~$\overline{f}_{\mathbf{C}}(X)=|\{K\in\mathbf{C}:K\bigcap X\neq\emptyset\}|$
\end{center}
are the lower and upper approximation numbers of $X$ with respect to $\mathbf{C}$, respectively.
We omit the subscript $\mathbf{C}$ when there is no confusion.
\end{definition}

Properties of the upper approximation number are studied in~\cite{WangZhu11Matroidal,WangZhuMin11Transversal,ZhuWang11Matroidal}.
Then, in this paper, we only investigate the properties of the lower approximation number and the relationship between it and the upper approximation number.

\begin{proposition}
\label{P:propertiesoflowernumber}
Let $\mathbf{C}$ be a covering on $U$ and $X, Y\subseteq U$.
The following properties hold:\\
(1) $\underline{f}(\emptyset)=0$;\\
(2) $\underline{f}(X)\leq\underline{f}(Y)$ if $X\subseteq Y$;\\
(3) $\underline{f}(X)+\underline{f}(Y)\leq\underline{f}(X\bigcup Y)+\underline{f}(X\bigcap Y)$;\\
(4) $\underline{f}(X)\leq\overline{f}(X)$;\\
(5) $\underline{f}(X)+\overline{f}(X^{c})=|\mathbf{C}|$.
\end{proposition}

\begin{proof}
(1) and (2) are straightforward.\\
(3): we need to prove $\underline{f}(X)+\underline{f}(Y)\leq\underline{f}(X\bigcup Y)+\underline{f}(X\bigcap Y)$, i.e., $|\{K\in\mathbf{C}:K\subseteq X\}|+|\{K\in\mathbf{C}:K\subseteq Y\}|\leq|\{K\in\mathbf{C}:K\subseteq X\bigcup Y\}|+|\{K\in\mathbf{C}:K\subseteq X\bigcap Y\}|$.
For all $S\in\{K\in\mathbf{C}:K\subseteq X\}$ or $S\in\{K\in\mathbf{C}:K\subseteq Y\}$, then $S\in\{K\in\mathbf{C}:K\subseteq X\bigcup Y\}$.
Similarly, for all $S\in\{K\in\mathbf{C}:K\subseteq X\}$ and $S\in\{K\in\mathbf{C}:K\subseteq Y\}$, then $S\in\{K\in\mathbf{C}:K\subseteq X\bigcap Y\}$.
This proves $\underline{f}(X)+\underline{f}(Y)\leq\underline{f}(X\bigcup Y)+\underline{f}(X\bigcap Y)$.\\
(4): According to Definition~\ref{D:lowerandupperapproximationnumbers}, $\underline{f}(X)\leq\overline{f}(X)$ is straightforward.\\
(5): We only need to prove $\underline{f}(X)+\overline{f}(X^{c})\leq|\mathbf{C}|$ and $\underline{f}(X)+\overline{f}(X^{c})\geq|\mathbf{C}|$.
For all $S\in\{K\in\mathbf{C}:K\subseteq X\}$, $S\bigcap X^{c}=\emptyset$, then $S\notin\{K\in\mathbf{C}:K\bigcap X^{c}\neq\emptyset\}$.
Similarly, for all $S\in\{K\in\mathbf{C}:K\bigcap X^{c}\neq\emptyset\}$, $S\nsubseteq X$, then $S\notin\{K\in\mathbf{C}:K\subseteq X\}$.
Therefore $|\{K\in\mathbf{C}:K\subseteq X\}|+|\{K\in\mathbf{C}:K\bigcap X^{c}\neq\emptyset\}|\leq |\mathbf{C}|$, i.e., $\underline{f}(X)+\overline{f}(X^{c})\leq|\mathbf{C}|$.
Conversely, $\underline{f}(X)+\overline{f}(X^{c})\geq|\mathbf{C}|$, we only need to prove $\{K\in\mathbf{C}:K\subseteq X\}\bigcup\{K\in\mathbf{C}:K\bigcap X^{c}\neq\emptyset\}\supseteq\mathbf{C}$, i.e., $\mathbf{C}-\{K\in\mathbf{C}:K\subseteq X\}\subseteq\{K\in\mathbf{C}:K\bigcap X^{c}\neq\emptyset\}$.
Suppose that there exists $S\in\mathbf{C}-\{K\in\mathbf{C}:K\subseteq X\}$ such that $S\bigcap X^{c}=\emptyset$, then $S\subseteq X$, which is contradictory with the condition.
Therefore, $\mathbf{C}-\{K\in\mathbf{C}:K\subseteq X\}\subseteq\{K\in\mathbf{C}:K\bigcap X^{c}\neq\emptyset\}$.
This proves $\underline{f}(X)+\overline{f}(X^{c})\geq|\mathbf{C}|$.
To sum up, this completes the proof.
\end{proof}

Partitions are a special kind of coverings.
In this paper, we propose partition-circuit matroids induced by partitions.
Therefore, in this paper, the lower and upper approximation numbers are based on partitions unless otherwise stated.

\subsection{Partition-circuit matroid through lower approximation number}
On a universe, a partition coincides with an equivalence relation.
According to Proposition~\ref{P:independentsetrepresentedbylower}, the lower approximation of any independent set in the partition-circuit matroid induced by a partition is equal to empty set with respect to the equivalence relation which coincides with the partition.
The independent sets of a partition-circuit matroid induced by a partition can be well expressed by the lower approximation number with respect to the partition.

\begin{proposition}
\label{P:independentsetrepresentlowernumber}
Let $\mathbf{P}$ be a partition on $U$ and $M_{\mathbf{P}}$ the partition-circuit matroid induced by $\mathbf{P}$.
Then, $\mathbf{I}_{\mathbf{P}}=\{X\subseteq U:\underline{f}(X)=0\}$.
\end{proposition}

\begin{proof}
According to Proposition~\ref{P:independentsetrepresentedbylower} and Definition~\ref{D:lowerandupperapproximationnumbers}, it is straightforward.
\end{proof}

The family of all circuits of the partition-circuit matroid induced by a partition is equal to the partition.
Then, all circuits can be represented by the lower approximation number with respect to the partition.

\begin{proposition}
Let $\mathbf{P}$ be a partition on $U$ and $M_{\mathbf{P}}$ the partition-circuit matroid induced by $\mathbf{P}$.
Then, $\mathbf{C}_{\mathbf{P}}=Min\{X\subseteq U:\underline{f}(X)=1\}$.
\end{proposition}

The rank function is to computer the cardinality of the maximal independent sets in subspaces.
Then, we use the lower approximation number to represent the rank function of the partition-circuit matroid induced by a partition.

\begin{proposition}
\label{P:rankandthelowerapproximation}
Let $\mathbf{P}$ be a partition on $U$ and $M_{\mathbf{P}}$ the partition-circuit matroid induced by $\mathbf{P}$.
Then, for all $X\subseteq U$,\\
\centerline{$r_{M_{\mathbf{P}}}(X)=|X|-\underline{f}(X)$.}
\end{proposition}

\begin{proof}
According to Proposition~\ref{P:rankdeterminematroid}, we only need to prove $\underline{f}(X)=0$ when $X\in\mathbf{I}_{\mathbf{P}}$.
According to Proposition~\ref{P:independentsetrepresentlowernumber}, it is straightforward.
\end{proof}

According to Definition~\ref{D:closure}, the closure of a subset is a set of all elements depending to the subset in matroids, in other words, the set of elements which are added to a subset, whose rank is equal to the rank of the subset, is the closure of the subset.
Then we represent the closure operator of partition-circuit matroids through the lower approximation number.

\begin{proposition}
\label{P:theclosureofpartition-circuitmatroid}
Let $\mathbf{P}$ be a partition on $U$ and $M_{\mathbf{P}}$ the partition-circuit matroid induced by $\mathbf{P}$.
Then, for all $X\subseteq U$,\\
\centerline{$cl_{M_{\mathbf{P}}}(X)=X\bigcup\{x\in X^{c}:\underline{f}(X\bigcup\{x\})=\underline{f}(X)+1\}$.}
\end{proposition}

\begin{proof}
According to Definition~\ref{D:closure},  $cl_{M_{\mathbf{P}}}(X)=\{x\in U:r_{M_{\mathbf{P}}}(X\bigcup\{x\})=r_{M_{\mathbf{P}}}(X)\}$.
Then, we only need to prove $\{x\in X^{c}:\underline{f}(X\bigcup\{x\})=\underline{f}(X)+1\}=\{x\in X^{c}:r_{M_{\mathbf{P}}}(X\bigcup\{x\})=r_{M_{\mathbf{P}}}(X)\}$.
According to Proposition~\ref{P:rankandthelowerapproximation}, $r_{M_{\mathbf{P}}}(X)=|X|-\underline{f}(X)$.
Therefore, for all $x\in X^{c}$, $r_{M_{\mathbf{P}}}(X\bigcup\{x\})=r_{M_{\mathbf{P}}}(X)\Leftrightarrow |X\bigcup\{x\}|-\underline{f}(X\bigcup\{x\})=|X|-\underline{f}(X)\Leftrightarrow |X|+1-\underline{f}(X\bigcup\{x\})=|X|-\underline{f}(X)\Leftrightarrow\underline{f}(X\bigcup\{x\})=\underline{f}(X)+1$.
\end{proof}

Definition~\ref{D:closedset} presents that a set is a closed set when the closure of the set is equal to itself.

\begin{corollary}
Let $\mathbf{P}$ be a partition on $U$, $M_{\mathbf{P}}$ the partition-circuit matroid induced by $\mathbf{P}$ and $X\subseteq U$.
$X$ is a closed set iff $\underline{f}(X)=\underline{f}(X\bigcup\{x\})$ for any $x\in U$.
\end{corollary}

\begin{proof}
According to Definition~\ref{D:closedset} and Proposition~\ref{P:theclosureofpartition-circuitmatroid}, it is straightforward.
\end{proof}

Some characteristics of partition-circuit matroids are well expressed by the lower approximation number based on partitions.

\subsection{duality of partition matroid through upper approximation number}

According to Proposition~\ref{P:partitoin-circuitmatroidispartitionone}, a partition-circuit matroid is a partition matroid.
The dual matroid of a partition matroid  can be expressed by the partition one in~\cite{LiuZhuZhang12Relationship}.
Through the upper approximation number~\cite{ZhuWang11Matroidal,WangZhu11Matroidal,WangZhuMin11Transversal} based on partitions in this paper, we investigate the dual matroids of partition-circuit matroids.
Firstly, we introduce a lemma to represent the relationship between a partition matroid and its dual matroid.

\begin{lemma}(\cite{LiuZhuZhang12Relationship})
\label{L:partitionmatroidanditsdual}
Let $\mathbf{P}=\{P_{1}, \cdots, P_{m}\}$ be a partition on $U$, $k_{1}, \cdots, k_{m}$ a group of nonnegative integers and $M(\mathbf{P}; k_{1}, \cdots, k_{m})$ the partition matroid induced by $\mathbf{P}$. Then
\begin{center}
$M^{*}(\mathbf{P}; k_{1}, \cdots, k_{m})=M(\mathbf{P}; |P_{1}|-k_{1}, \cdots, |P_{m}|- k_{m})$.
\end{center}
\end{lemma}

The dual matroid of a partition matroid is also a partition one.
Then, we can obtain the following proposition.

\begin{proposition}
\label{p:dualmatroidofpartionmatroid}
Let $\mathbf{P}$ be a partition on $U$, $M_{\mathbf{P}}$ the partition-circuit matroid induced by $\mathbf{P}$ and $M_{\mathbf{P}}^{*}$ its dual matroid.
Then, $\mathbf{I}_{\mathbf{P}}^{*}=\{X\subseteq U:\forall P\in\mathbf{P}, |X\bigcap P|\leq 1\}$
\end{proposition}

\begin{proof}
According to Definition~\ref{D:partitionmatroid}, Proposition~\ref{P:representedbypartitionmatroid} and Lemma~\ref{L:partitionmatroidanditsdual}, it is straightforward.
\end{proof}

The dual matroid of a partition-circuit matroid can be expressed by the elements of the partition.
Then, the independent sets of the dual matroid of a partition-circuit matroid are investigated through the upper approximation number.

\begin{proposition}
\label{P:independentbyuppernumber}
Let $\mathbf{P}$ be a partition on $U$, $M_{\mathbf{P}}$ the partition-circuit matroid induced by $\mathbf{P}$ and $M_{\mathbf{P}}^{*}$ its dual matroid.
Then, $\mathbf{I}_{\mathbf{P}}^{*}=\{X\subseteq U:\overline{f}(X)=|X|\}$.
\end{proposition}

\begin{proof}
According to Proposition~\ref{p:dualmatroidofpartionmatroid}, we only need to prove $\{X\subseteq U:\forall P\in\mathbf{P}, |X\bigcap P|\leq 1\}=\{X\subseteq U:\overline{f}(X)=|X|\}$.
For all $X\subseteq U$, $|X|=|X\bigcap U|=|X\bigcap(\bigcup_{P\in\mathbf{P}}P)|=\sum_{P\in\mathbf{P}}|X\bigcap P|$.
$\forall P\in\mathbf{P}, |X\bigcap P|\leq 1\Leftrightarrow\overline{f}(X)=\sum_{P\in\mathbf{P}}|X\bigcap P|=|X|$.
To sum up, this completes the proof.
\end{proof}

One of equivalent definitions of a matroid is rank function.
And a matroid and its rank function determine each other.
The rank function of the dual matroids of partition-circuit ones can be well expressed by the upper approximation number.

\begin{proposition}
\label{P:rankrepresentbyuppernumber}
Let $\mathbf{P}$ be a partition on $U$, $M_{\mathbf{P}}$ the partition-circuit matroid induced by $\mathbf{P}$ and $M_{\mathbf{P}}^{*}$ its dual matroid.
Then, for all $X\subseteq U$, $r_{M_{\mathbf{P}}}^{*}(X)=\overline{f}(X)$.
\end{proposition}

\begin{proof}
According to Proposition~\ref{P:relationshipbetweenmatroidanditsdual}, $r^{*}_{M_{\mathbf{P}}}(X)=|X|-r_{M_{\mathbf{P}}}(U)+r_{M_{\mathbf{P}}}(X^{c})$.
According to (5) of Proposition~\ref{P:propertiesoflowernumber}, $\underline{f}(X^{c})+\overline{f}(X)=|\mathbf{P}|$.
According to Proposition~\ref{P:rankandthelowerapproximation}, $r_{M_{\mathbf{P}}}(X)=|X|-\underline{f}(X)$.
Therefore, for all $X\subseteq U$, $r_{M_{\mathbf{P}}}^{*}(X)=|X|-r_{M_{\mathbf{P}}}(U)+r_{M_{\mathbf{P}}}(X^{c})=|X|-(|U|-\underline{f}(U))+(|X^{c}|-\underline{f}(X^{c}))=|X|+|X^{c}|-|U|+\underline{f}(U)-\underline{f}(X^{c})=\underline{f}(U)-\underline{f}(X^{c})=|\mathbf{P}|-\underline{f}(X^{c})=\overline{f}(X)$.
\end{proof}

Added an element to a set, the rank of the set does not change, then the element belongs to the closure of the set.
According to Proposition~\ref{P:rankrepresentbyuppernumber}, the closure operator of a matroid can be represented through the upper approximation number.

\begin{proposition}
Let $\mathbf{P}$ be a partition on $U$, $M_{\mathbf{P}}$ the partition-circuit matroid induced by $\mathbf{P}$ and $M_{\mathbf{P}}^{*}$ its dual matroid.
Then, for all $X\subseteq U$,\\
\centerline{$cl_{M_{\mathbf{P}}}^{*}(X)=\{x\in U:\overline{f}(X)=\overline{f}(X\bigcup\{x\})\}$.}
\end{proposition}

\begin{proof}
According to Definition~\ref{D:closure} and Proposition~\ref{P:rankrepresentbyuppernumber}, it is straightforward.
\end{proof}

\section{Conclusions}
\label{S:conclusions}
In this paper, we proposed partition-circuit matroids induced by partitions, and studied some characteristics of them and their dual matroids.
Similar to the upper approximation number, we proposed the lower approximation number with respect to a covering.
Then the properties of the lower approximation number and the relationships between it and the upper approximation number are studied.
Partitions are a special kind of coverings.
In this paper, we used the lower and upper approximation numbers, which are based on partitions.
Some characteristics of partition-circuit matroids and the dual matroids of them are investigated through the lower and upper approximation numbers.
In future work, we will study the connection between covering-based rough sets and matroids through the lower and upper approximation numbers based on coverings.

\section{Acknowledgments}
This work is supported in part by the National Natural Science Foundation of China under Grant No. 61170128, the Natural Science Foundation of Fujian Province, China, under Grant Nos. 2011J01374 and 2012J01294, the Science and Technology Key Project of Fujian Province, China, under Grant No. 2012H0043 and State key laboratory of management and control for complex systems open project under Grant No. 20110106.



\end{document}